\documentclass[11pt]{article}

\usepackage{epsfig,epsf,fancybox}
\usepackage{amsmath}
\usepackage{amsthm}
\usepackage{mathrsfs}
\usepackage{amssymb}
\usepackage{graphicx}
\usepackage{color}
\usepackage{multirow}
\usepackage{paralist}
\usepackage{verbatim}
\usepackage{galois}
\usepackage{algorithm}
\usepackage{algorithmic}
\usepackage{boxedminipage}
\usepackage{booktabs}
\usepackage{accents}
\usepackage{stmaryrd}
\usepackage{subfig}

\newtheorem{theorem}{Theorem}

\newtheorem{remark}[theorem]{Remark}
\newtheorem{assumption}[theorem]{Assumption}
\usepackage{thmtools,thm-restate}

\usepackage{bm}
\usepackage{mwe}

\newcommand{\be}{\begin{equation}}
\newcommand{\ee}{\end{equation}}
\newcommand{\ba}{\begin{array}}
\newcommand{\ea}{\end{array}}

\newcommand{\br}{\mathbb{R}}

\newcommand{\ex}{{\bf\mathbb{E}}}

\newcommand{\XCal}{\mathcal{X}}
\newcommand{\YCal}{\mathcal{Y}}
\newcommand{\etal}{\textit{{et al. }}}

\newcommand{\prox}{\textbf{\rm prox}}

\newcommand{\st}{\textnormal{s.t.}}

\newcommand{\sign}{\textnormal{sign}}

\,
\,

\newcommand{\argmin}{\mathop{\rm argmin}}

\newcommand{\LCal}{\mathcal{L}}

\newcommand{\bpm}{\begin{pmatrix}}
\newcommand{\epm}{\end{pmatrix}}

\title{Stochastic Primal-Dual Proximal ExtraGradient Descent for Compositely Regularized Optimization}

\author{
Tianyi Lin
\thanks{Department of Industrial Engineering and Operations Research, UC Berkeley, California, USA. Email: darren\_lin@berkeley.edu}
\and
Linbo Qiao
\thanks{College of Computer, National University of Defense Technology, Changsha, China.} 
\and
Teng Zhang
\thanks{Department of Management Science and Engineering, Stanford University, USA.} 
\and
Jiashi Feng
\thanks{National University of Singapore, Singapore.} 
\and
Bofeng Zhang$^{\dagger}$ 
}

\begin{document}
\maketitle

\begin{abstract}
We consider a wide range of regularized stochastic minimization problems with two regularization terms, one of which is composed with a linear function. This optimization model abstracts a number of important applications in artificial intelligence and machine learning, such as fused Lasso, fused logistic regression, and a class of graph-guided regularized minimization. The computational challenges of this model are in two folds. On one hand, the closed-form solution of the proximal mapping associated with the composed regularization term or the expected objective function is not available. On the other hand, the calculation of the full gradient of the expectation in the objective is very expensive when the number of input data samples is considerably large. To address these issues, we propose a stochastic variant of extra-gradient type methods, namely \textsf{Stochastic Primal-Dual Proximal ExtraGradient descent (SPDPEG)}, and analyze its convergence property for both convex and strongly convex objectives. For general convex objectives, the uniformly average iterates generated by \textsf{SPDPEG} converge in expectation with $O(1/\sqrt{t})$ rate. While for strongly convex objectives, the uniformly and non-uniformly average iterates generated by \textsf{SPDPEG} converge with $O(\log(t)/t)$ and $O(1/t)$ rates, respectively. The order of the rate of the proposed algorithm is known to match the best convergence rate for first-order stochastic algorithms. Experiments on fused logistic regression and graph-guided regularized logistic regression problems show that the proposed algorithm performs very efficiently and consistently outperforms other competing algorithms.
\end{abstract}


\section{Introduction}
In this paper, we are interested in solving a class of convex optimization problems with both non-composite and composite regularization terms:
\begin{eqnarray}\label{prob}
\min\limits_{x\in\XCal} \ \ex_{\xi}\left[l(x,\xi)\right] + r_1(x) + r_2(Fx),
\end{eqnarray}
where $\XCal\subset\br^d$ is a convex compact set with diameter $D_x$, the regularization terms $r_1: \br^d\rightarrow\br$ and $r_2: \br^l\rightarrow\br$ are both convex but possibly nonsmooth, and $r_2$ is composed with a possibly non-diagonal penalty matrix $F\in\br^{l\times d}$ specifying the desired structured sparsity pattern in $x$. We denote $l(\cdot,\cdot): \br^d\times\Omega\rightarrow\br$ as a convex and smooth loss function of a rule $x$ for a sample data $\left\{\xi_i = \left(a_i, b_i\right)\right\}$, and define the corresponding expectation as $l(x) = \ex_{\xi}\left[l(x,\xi)\right]$.

The above formulation covers quite a few popular models arising from statistics and machine learning, such as Lasso \cite{Tibshirani-1996-Regression} obtained by setting $l(x,\xi_i)=\frac{1}{2}\left\|a_i^\top x - b_i\right\|^2$ and $r_1(x)=\lambda\left\|x\right\|_1$ and $r_2=0$, and linear SVM \cite{Cortes-1995-Support} obtained by letting $l(x,\xi_i)=\max\left(0, 1-b_i \cdot a_i^\top x\right)$ and $r_1(x)=(\lambda/2)\left\|x\right\|_2^2$ and $r_2=0$, where $\lambda>0$ is a parameter. More importantly, we can accommodate problem~\eqref{prob} with more complicated structures by imposing the non-trivial regularization term $r_2(Fx)$, such as fused Lasso \cite{Tibshirani-2005-Sparsity}, fused logistic regression and graph-guided regularized minimization \cite{Friedman-2009-Elements}.

The standard algorithm applied to solve problem~\eqref{prob} is proximal gradient descent \cite{Parikh-2014-Proximal}.  However, there are two main difficulties: 1) Computing the exact proximal gradient is intractable since the closed-form solution to the proximal mapping of $r_1(x)+r_2(Fx)$, or even single $r_2(Fx)$ is in usually unavailable; 2) the computational complexity of the full gradient $\nabla l(x)$ rapidly increases as the size of samples grows, and is hence prohibitively expensive for modern data-intensive applications.

A common way to suppress the former one is to introduce a new auxiliary variable $z$ with $z=Fx$ and reformulate problem~\eqref{prob} as a linearly constrained convex problem with respect to two variables $x$ and $z$ as follows:
\begin{eqnarray}\label{prob-1}
\min\limits_{x\in\XCal} & \ex_{\xi}\left[l(x,\xi)\right] + r_1(x) + r_2(z), \nonumber\\
 \st & \ Fx - z = 0.
\end{eqnarray}
Then one can resort to \textsf{Linearized Alternating Direction Method of Multipliers (LADMM)} \cite{Boyd-2011-Distributed, Yang-2013-Linearized}. Very recently, Lin \etal \cite{Lin-2015-Extragradient} have explored the efficiency of the extra-gradient descent \cite{Korpelevich-1976-Extragradient, Korpelevich-1983-Extragradient}, and further showed the hybrid Extra-Gradient ADM (\textsf{EGADM}) is very efficient on moderate size problems. However, these methods are computationally expensive due to the computation of the full gradient in each iteration.

To address the computational issue, several stochastic ADMM algorithms \cite{Ouyang-2013-Stochastic,Suzuki-2013-Dual, Azadi-2014-Towards, Zhao-2015-Adaptive, Zheng-2016-Stochastic} have been proposed. The idea is to draw a mini-batch of samples and then compute a noisy sub-gradient of $l(x) + r_1(x)$ on the mini-batch in each iteration. However, for problem~\eqref{prob} with non-smooth regularization (which is actually common in practice), these sub-gradient type alternating direction methods may be slow and unstable \cite{Duchi-2009-Efficient}.

In this work, we propose a \textsf{Stochastic Primal-Dual Proximal Extra-Gradient Descent (SPDPEG)}, which inherits the advantages of \textsf{EGADM} and stochastic methods. Basically, the proposed method computes two noisy gradients of $l$ at the $k$-th iteration by randomly drawing two data samples $\xi_1^{k+1}$ and $\xi_2^{k+1}$, and then performs extra-gradient descent along the noisy gradients. We demonstrate that the proposed algorithm is very efficient and stable in solving problem~\eqref{prob} with possible non-smooth terms at large scale.

\textbf{Our contribution:} We propose a novel \textsf{Stochastic Primal-Dual Proximal Extra-Gradient Descent (SPDPEG)}. \textsf{SPDPEG} is efficient in solving large-scale problems with composite and nonsmooth regularizations. We demonstrate its theoretical convergence for both convex and strongly convex objectives. For convex objectives, \textsf{SPDPEG} has the convergence rate of $O(1/\sqrt{t})$ in expectation with the uniformly average iterates. This convergence rate is known to be the best possible for minimizing general convex objective using first-order noisy oracle\cite{Agarwal-2012-Information}. When the objective to be optimized is strongly convex, \textsf{SPDPEG} converges at the rates of $O(\log(t)/t)$ and $O(1/t)$ in expectation with the uniformly and non-uniformly average iterates, respectively. This matches the convergence rate of stochastic ADMM with a significantly stronger robustness in terms of the numerical performance, as confirmed by encouraging experiments on fused logistic regression and graph-guided regularized minimization tasks.

\textbf{Related work:} The first line of related work are various stochastic alternating direction methods \cite{Ouyang-2013-Stochastic,Suzuki-2013-Dual,Azadi-2014-Towards,Gao-2014-Information,Zhong-2014-Fast,Suzuki-2014-Stochastic,Zhao-2015-Adaptive, Zheng-2016-Stochastic} developed to solve problem~\eqref{prob-1}. They fall into two camps: 1) to compute the noisy sub-gradient of $l+r_1$ on a mini-batch of data samples and perform sub-gradient descent \cite{Ouyang-2013-Stochastic, Suzuki-2013-Dual, Azadi-2014-Towards, Gao-2014-Information, Zhao-2015-Adaptive}; 2) to approximate problem~\eqref{prob-1} using the finite-sum loss and perform variance-reduced gradient descent or dual coordinate ascent \cite{Zhong-2014-Fast, Suzuki-2014-Stochastic, Zheng-2016-Stochastic}. 

For the first group of algorithms, drawing a noisy sub-gradient may lead to the unstable numerical performance, especially on large-scale problems. In the experimental section, we compare our algorithm against \textsf{SGADM} \cite{Gao-2014-Information} and demonstrate the significant improvement. 

For the second group of algorithms, it is not always feasible to use the finite-sum loss since we know nothing about the underlying distribution of data. In specific, Zhong and Kwok \cite{Zhong-2014-Fast} proposed a \textsf{Stochastic Averaged Gradient-based ADM (SAG-ADM)} whose iteration complexity is $O(1/t)$. However, \textsf{SAG-ADM} needs to store a few variables and incurs a very high memory cost. Suzuki \cite{Suzuki-2014-Stochastic} proposed a linearly convergent \textsf{Stochastic Dual Coordinate Ascent ADM (SDCA-ADM)}. However, a stronger assumption on $r_1$ and $r_2$ such as strong convexity and smoothness is imposed. Zheng and Kwok \cite{Zheng-2016-Stochastic} proposed a \textsf{Stochastic Variance-Reduce Gradient-based ADM (SVRG-ADM)} for convex and non-convex problems. However, \textsf{SVRG-ADM} only focuses on the finite-sum problem. In contrast, our \textsf{SPDPEG} approach can be applied to solve problem~\eqref{prob-1} in very general form. 

Very recently, a stochastic variant of hybrid gradient method, namely \textsf{SPDHG} \cite{Qiao-2016-Stochastic}, has been proposed to solve a class of compositely regularized minimization problems with very \textbf{special} regularization. In specific, $r_1\equiv 0$ and $r_2(x) = \max\limits_{y \in \YCal} \left\langle y,x\right\rangle$ (See Assumption 3 in \cite{Qiao-2016-Stochastic}). However, such assumption is very strong and does not hold for many compositely regularized minimization problems. This motivates us to consider problem~\eqref{prob-1} and develop \textsf{SPDPEG} approach.

The second line of related works is various extra-gradient methods. This idea is not new and originally proposed by Korpelevich for solving saddle-point problems and variational inequalities \cite{Korpelevich-1976-Extragradient, Korpelevich-1983-Extragradient}. The convergence and iteration complexity of extra-gradient methods are established in  \cite{Noor-2003-New} and \cite{Nemirovski-2004-Prox} respectively. There are also some variants of extra-gradient methods. Solodov and Svaiter proposed a hybrid proximal extra-gradient method \cite{Solodov-1999-Hybrid}, whose iteration complexity is established by Monteiro and Svaiter in \cite{Monteiro-2010-Complexity, Monteiro-2011-Complexity, Monteiro-2013-Iteration}. Bonettini and Ruggiero studied a generalized extragradient method for total variation based image restoration problem \cite{Bonettini-2011-Alternating}. To the best of our knowledge, this is the first time that a stochastic primal-dual variant of extra-gradient type methods is introduced to solve problem~\eqref{prob}.
\section{Problem Set-Up and Methods}
We make the following assumptions that are common in optimization literature and usually hold in practice throughout the paper:
\begin{assumption}\label{assumption_1}
The optimal set of problem~\eqref{prob} is nonempty.
\end{assumption}
\begin{assumption}\label{assumption_2}
$l(\cdot)$ is continuously differentiable with Lipschitz continuous gradient. That is, there exists a constant $L>0$ such that
\begin{displaymath}
\left\|\nabla l(x_1) - \nabla l(x_2)\right\|\leq L\left\| x_1 - x_2\right\|, \forall x_1, x_2\in\XCal.
\end{displaymath}
\end{assumption}
Assumption~\ref{assumption_2} holds for many problems in machine learning. For example, the following least squares and logistic functions are two standard ones:
\begin{displaymath}
l(x,\xi_i) = \frac{1}{2}\left\| a_i^\top x - b_i\right\|^2 \mbox{ or } \ l(x,\xi_i) = \log\left( 1+\exp\left(-b_i \cdot a_i^\top x\right) \right),
\end{displaymath}
where $\xi_i = \left(a_i,b_i\right)$ is a single data sample.

\begin{assumption}\label{assumption_3}
The regularization functions $r_1$ and $r_2$ are both continuous but possibly non-smooth; the associated proximal mapping for each individual regularization admits a closed-form solution, \textit{i.e.},
\begin{equation}
\prox_{r_i}(x) = \argmin_y \ r_i(y) + \frac{1}{2}\left\| y-x\right\|_2^2
\end{equation}
can be calculated in a closed form for $i=1,2$.
\end{assumption}

\begin{remark}
We remark that Assumption \ref{assumption_3} is reasonable for a class of optimization problems regularized by $\ell_1$-norm or nuclear norm, such as fused Lasso, fused logistic regression, and graph-guided regularized minimization problems. The proximal mapping of $\ell_1$-norm can be computed as follows:
\begin{eqnarray*}
\left[\prox_{\left\|\cdot\right\|_1}(x)\right]_i & = & \argmin_y \ \left\|y\right\|_1 + \frac{1}{2}\left\| y-x_i\right\|_2^2 \\
& = & \left\{ \begin{array}{cc}
\sign(x_i)(|x_i|-1) & |x_i|>1, \\
0 & |x_i|\leq 1.
\end{array}\right.
\end{eqnarray*}
We clarify that the proximal mapping of $r(x)$ and that of $r(Fx)$ are totally different and have different properties. For example, the proximal mapping of $\|x\|_1$ admits a closed-form solution but the proximal mapping of $\|Fx\|_1$ does not admit in general when $F$ is non-diagonal. We only assume that the proximal mapping of $r(x)$ admits a closed-form solution in Assumption \ref{assumption_3} but expect to address the case of $r(Fx)$ whose proximal mapping does not admit a closed-form solution in general.
\end{remark}

\begin{assumption}\label{assumption_4}
The gradient of the objective function $l(x)$ is easy to estimate. Any stochastic gradient estimation $\nabla l(\cdot,\xi)$ for $\nabla l(\cdot)$ at $x$ satisfies
\begin{displaymath}
\ex_{\xi}\left[\nabla l(x,\xi)\right] = \nabla l(x),
\end{displaymath}
and
\begin{displaymath}
\ex_{\xi}\left[\left\| \nabla l(x,\xi) - \nabla l(x)\right\|^2\right] \leq \sigma^2,
\end{displaymath}
where $\sigma>0$ is a constant number.
\end{assumption}

\begin{assumption}\label{assumption_5}
$l(\cdot)$ is \textbf{$\mu$-strongly convex} at $x$. In other words, there exists a constant $\mu>0$ such that
\begin{displaymath}
l(y) - l(x) - \left(y-x\right)^\top\nabla l(x) \geq \frac{\mu}{2}\left\| y-x\right\|^2, \forall y\in\XCal.
\end{displaymath}
\end{assumption}
We remark that our algorithm works even without Assumption \ref{assumption_5}. However, the lower iteration complexity will be obtained with Assumption \ref{assumption_5}. 

\begin{algorithm}[t]
\caption{Stochastic Primal-Dual Proximal ExtraGradient  (SPDPEG)}\label{ALG:SPDPEG}
\begin{algorithmic}
\STATE \textbf{Initialize:} $x^0$, $z^0$, and $\lambda^0$.
\FOR {$k = 0,1,2,\cdots$}
\STATE choose two data samples $\xi_1^{k+1}$ and $\xi_2^{k+1}$ randomly;
\STATE update $z^{k+1}$ according to Eq.~\eqref{Eq:Update_Z};
\STATE update $\bar{x}^{k+1}$ according to Eq.~\eqref{Eq:Update_barX};
\STATE update $\bar{\lambda}^{k+1}$ according to Eq.~\eqref{Eq:Update_barLambda};
\STATE update $x^{k+1}$ according to Eq.~\eqref{Eq:Update_X};
\STATE update $\lambda^{k+1}$ according to Eq.~\eqref{Eq:Update_Lambda};
\ENDFOR
\STATE \textbf{Output:} $\tilde{z}^t = \sum\limits_{k=0}^t \alpha^{k+1} z^{k+1}$, $\tilde{x}^t = \sum\limits_{k=0}^t \alpha^{k+1} \bar{x}^{k+1}$, and $\tilde{\lambda}^t = \sum\limits_{k=0}^t \alpha^{k+1} \bar{\lambda}^{k+1}$.
\end{algorithmic}
\end{algorithm}

We introduce the \textsf{Stochastic Primal-Dual Proximal ExtraGradient (SPDPEG)} method, and further discuss the choice of step-size. We define the augmented Lagrangian function for problem~\eqref{prob-1} as
\[ \LCal_\gamma\left(z,x,\lambda\right) = r_2(z) + r_1(x) + \phi\left(z,x,\lambda\right) + \frac{\gamma}{2}\left\|Fx - z\right\|^2, \]
where $\lambda\in\br^p$ is the dual variable associated with $Fx=z$. $\phi$ is defined as
\[ \phi\left(z,x,\lambda\right) = l(x) - \left\langle\lambda, Fx - z\right\rangle. \]
The \textsf{SPDPEG} algorithm is based on the \textbf{primal-dual update scheme} where $(z,x)$ is primal variable and $\lambda$ is a dual variable, and can be seen as an inexact augmented Lagrangian method. The details are presented in Algorithm \ref{ALG:SPDPEG}. 

We provide details on following four important issues: how to solve the primal and dual sub-problems easily, how to apply the noisy gradient and perform extra-gradient descent, how to choose step-size, and how to determine the weights for the non-uniformly average iterates.
\begin{enumerate}
\item\textbf{Update for $z$:} The first sub-problem in Algorithm \ref{ALG:SPDPEG} is to minimize the augmented Lagrangian function $\LCal_\gamma$ with respect to $z$, \textit{i.e.},
\begin{equation}\label{Eq:Update_Z}
z^{k+1} := \displaystyle \argmin_z \ \LCal_\gamma(z, x^k; \lambda^k),
\end{equation}
which is equivalent to computing the proximal mapping of $r_2$ and hence admits a closed-form solution from Assumption \ref{assumption_3}.

\item\textbf{Stochastic Gradient:} According to Assumption \ref{assumption_4}, $\phi$ is known to be easy for gradient estimation with respect to $x$, and the stochastic gradient estimation $G\left(z,x,\lambda;\xi\right)$ is defined as
\[ G\left(z,x,\lambda;\xi\right) = \nabla l(x,\xi) - F^\top\lambda.\]
To update $x$, the SPDPEG algorithm takes a proximal extra-gradient step using a stochastic gradient estimation $G\left(z,x,\lambda;\xi\right)$ and different step-sizes, \textit{i.e.},
\begin{align}
\bar{x}^{k+1} & := \prox_{c^{k+1} r_1}\left( x^k - c^{k+1}G\left(y^{k+1}, x^k, \lambda^k; \xi_1^{k+1}\right)\right), \label{Eq:Update_barX} \\
\bar{\lambda}^{k+1} & := \lambda^k -\gamma\left( Fx^k - z^{k+1} \right),
\label{Eq:Update_barLambda}
\end{align}
and
\begin{align}
x^{k+1} & := \prox_{c^{k+1} r_1}\left(x^k - c^{k+1}G\left(y^{k+1}, \bar{x}^{k+1},\bar{\lambda}^{k+1}; \xi_2^{k+1}\right)\right), \label{Eq:Update_X} \\
\lambda^{k+1} & := \lambda^k -\gamma\left(F\bar{x}^{k+1} - z^{k+1} \right). \label{Eq:Update_Lambda}
\end{align}

\item\textbf{Step-Size $c^{k+1}$:} The choice of step-size $c^{k+1}$ depends on whether the objective function is strongly convex or not. The rate of convergence varies with respect to different step-size rules. Moreover, a sequence of vanishing step-sizes is necessary since we do not adopt any technique of variance reduction in the proposed algorithm.

\item\textbf{Non-Uniformly Average Iterates:}
\cite{Azadi-2014-Towards} showed that non-uniform average iterates generated by stochastic algorithms converge with fewer iterations. Inspired by this work, through non-uniformly averaging the iterates of the SPDPEG algorithm and adopting a slightly modified step-size, we manage to establish an accelerated convergence rate of $O(1/t)$ in expectation.
\end{enumerate}
\section{Main Result}
In this section, we present the main result in this paper. For general convex objectives, the uniformly average iterates generated by the \textsf{SPDPEG} algorithm converge in expectation with $O(1/\sqrt{t})$ rate. While for strongly convex objectives, the uniformly and non-uniformly average iterates generated converge in expectation with $O(\log(t)/t)$ and $O(1/t)$ rates, respectively. \textbf{The computational complexity are $O(d/\sqrt{t})$, $O(d\log(t)/t)$ and $O(d/t)$ since the per-iteration complexity is the computational cost of the noisy gradient on $\xi_1$ and $\xi_2$ and the proximal mapping}, where $d$ is the dimension of decision variable. The main theoretic results with respect to different settings are summarized as follows:
\begin{enumerate}
\item Assuming that $l$ is a general convex objective function, the step-size is $c^{k+1} = \frac{1}{\sqrt{k+1} + \tilde{L}}$, and the weight of the iterates is $\alpha^{k+1} = \frac{1}{t+1}$, the proposed SPDPEG algorithm converges with the $O(1/\sqrt{t})$ rate in expectation.
\item Assuming that $l$ is a $\mu$-strongly convex objective function, the step-size is $c^{k+1} = \frac{2}{\mu\left(k+1\right) + 2\tilde{L}}$, and the weight of the iterates is $\alpha^{k+1} = \frac{1}{t+1}$, the proposed SPDPEG algorithm converges with the $O(\log(t)/t)$ rate in expectation.
\item Assuming that $l$ is a $\mu$-strongly convex objective function, the step-size is $c^{k+1} = \frac{4}{\mu\left(k+2\right) + 4\tilde{L}}$, the weight of the iterates is $\alpha^{k+1} = \frac{2(k+3)}{(t+1)(t+6)}$, and the dual variables are bounded by $D_\lambda>0$ (this assumption is standard and also adopted in \cite{Azadi-2014-Towards}), the proposed SPDPEG algorithm converges with the $O(1/t)$ rate in expectation.
\end{enumerate}
In the above, $\tilde{L}$ is defined as
\begin{displaymath}
\tilde{L} = \max\left\{8\gamma \sigma_{\max}(F^\top F) + \mu, \sqrt{8L^2 + \gamma\sigma_{\max}(F^\top F)} + \mu\right\},
\end{displaymath}
where $\sigma_{\max}(F^\top F)$ denotes the largest eigenvalue of $F^\top F$, and $\mu=0$  when $l$ is a general convex objective function.

We present the main theoretic result for uniformly average iterates under general convex objective functions in the following theorem.
\begin{restatable}{theorem}{pTheoremConvexAverage}
\label{Theorem-Convex-Average}
Consider the SPDPEG algorithm with uniformly average iterates. For any
optimal solution $\left(z^*,x^*\right)$, it holds that
\begin{align}
\left| \ex[l(\tilde{x}^t)] + \ex[r_1(\tilde{x}^t)] + \ex[r_2(\tilde{z}^t)] - l(x^*) - r_1(x^*) - r_2(z^*)\right| & = O(1/\sqrt{t}), \label{result-convex-average-1} \\
\left\| F\ex[\tilde{x}^t] - \ex[\tilde{z}^t] \right\| & = O(1/\sqrt{t}). \label{result-convex-average-2}
\end{align}
Note that this implies that the SPDPEG algorithm converges in expectation with the $O(1/\sqrt{t})$ rate in terms of both the objective error and constraint violation.
\end{restatable}

We present the main theoretic result for uniformly average iterates under a strongly convex objective function in the following theorem.
\begin{restatable}{theorem}{pTheoremStronglyConvexAverage}
\label{Theorem-Strongly-Convex-Average}
Consider the SPDPEG algorithm with uniformly average iterates. For any optimal solution $\left(z^*, x^*\right)$, it holds that
\begin{align}
\left| \ex[l(\tilde{x}^t)] + \ex[r_1(\tilde{x}^t)] + \ex[r_2(\tilde{z}^t)] - l(x^*) - r_1(x^*) - r_2(z^*)\right| & = O(\log(t)/t), \label{result-strongly-convex-average-1} \\
\left\| F\ex[\tilde{x}^t] - \ex[\tilde{z}^t] \right\| & = O(\log(t)/t). \label{result-strongly-convex-average-2}
\end{align}
Note that this implies that the SPDPEG algorithm converges in expectation with the $O(\log(t)/t)$ rate in terms of both the objective error and constraint violation.
\end{restatable}

We present the main theoretic result for non-uniformly average iterates under a strongly convex objective function in the following theorem.
\begin{restatable}{theorem}{pTheoremStronglyConvexNonaverage}
\label{Theorem-Strongly-Convex-Nonaverage}
Consider the SPDPEG algorithm with non-uniformly average iterates. For any
optimal solution $\left(z^*, x^*\right)$, it holds that
\begin{align}
\left| \ex[l(\tilde{x}^t)] + \ex[r_1(\tilde{x}^t)] + \ex[r_2(\tilde{z}^t)] - l(x^*) - r_1(x^*) - r_2(z^*)\right| & = O(1/t), \label{result-strongly-convex-nonaverage-1} \\
\left\| F\ex[\tilde{x}^t] - \ex[\tilde{z}^t] \right\| & = O(1/t). \label{result-strongly-convex-nonaverage-2}
\end{align}
Note that this implies that the SPDPEG algorithm converges in expectation with the $O(1/t)$ rate in terms of both the objective error and constraint violation.
\end{restatable}

\section{Proof}
We first prove the key technical lemma which is very important to the proof of Theorem \ref{Theorem-Convex-Average}-Theorem \ref{Theorem-Strongly-Convex-Nonaverage}. 
\begin{restatable}{lemma}{pLemmaUnified}
\label{Lemma-Unified}
The sequence $\left\{z^{k+1}, \bar{x}^{k+1},\bar{\lambda}^{k+1}, x^{k+1}, \lambda^{k+1}\right\}$ generated by the SPDPEG algorithm satisfies the following inequality:
\begin{eqnarray}\label{inequality-noisy-optimality}
& & r_1(x) + r_2(z) - r_1(\bar{x}^{k+1}) - r_2(z^{k+1}) + \left( \begin{array}{c} z - z^{k+1} \\ x - \bar{x}^{k+1} \\ \lambda - \bar{\lambda}^{k+1} \end{array} \right)^\top\left(\begin{array}{c} \bar{\lambda}^{k+1} \\ G\left( z^{k+1}, \bar{x}^{k+1}, \bar{\lambda}^{k+1}; \xi_2^{k+1}\right) \\  F\bar{x}^{k+1} - z^{k+1}\end{array}\right) \nonumber \\
& \geq & \frac{1}{2c^{k+1}}\left\| x - x^{k+1} \right\|^2 - \frac{1}{2c^{k+1}}\left\| x - x^k \right\|^2 - 4c^{k+1}\left\|\delta^{k+1}\right\|^2 - 4c^{k+1}\left\|\bar{\delta}^{k+1}\right\|^2 \nonumber \\
& & -\frac{1}{2\gamma}\left\|\lambda - \lambda^k\right\|^2 + \frac{1}{2\gamma}\left\|\lambda - \lambda^{k+1}\right\|^2 + \left[ \frac{1}{2\gamma} -4c^{k+1}\sigma_{\max}(F^\top F)\right] \left\| \lambda^k - \bar{\lambda}^{k+1}\right\|^2 \nonumber \\
& & + \left[ \frac{1}{2c^{k+1}} - \frac{\gamma\sigma_{\max}(F^\top F)}{2} - 4c^{k+1}L^2\right]\left\| x^k - \bar{x}^{k+1}\right\|^2 + \frac{1}{2c^{k+1}}\left\| x^{k+1} - \bar{x}^{k+1} \right\|^2,
\end{eqnarray}
where $\delta^{k+1}$ and $\bar{\delta}^{k+1}$ are respectively denoted by
\begin{equation}\label{def-delta}
\delta^{k+1} = \nabla l(x^k, \xi_1^{k+1}) - \nabla l(x^k) \quad \textrm{and} \quad  \bar{\delta}^{k+1} = \nabla l(\bar{x}^{k+1}, \xi_2^{k+1}) - \nabla l(\bar{x}^{k+1}).
\end{equation}
\end{restatable}
\begin{proof}
The first-order optimality condition for updating $z^{k+1}$ is given by
\begin{equation}\label{inequality-Z-gap}
r_2(z) - r_2(z^{k+1}) + \left\langle z - z^{k+1}, \bar{\lambda}^{k+1}\right\rangle \geq 0.
\end{equation}
For $\bar{x}^{k+1}, x^{k+1}\in\XCal$ and any $x\in\XCal$, the first-order optimality condition for updating $\bar{x}^{k+1}$ and $x^{k+1}$ are given respectively by
\begin{align}
r_1(x) - r_1(\bar{x}^{k+1}) + \left\langle x - \bar{x}^{k+1}, \frac{\bar{x}^{k+1} - x^k}{c^{k+1}} + G\left(z^{k+1}, x^k, \lambda^k; \xi_1^{k+1}\right) \right\rangle & \geq 0, \label{opt-barX} \\
r_1(x) - r_1(x^{k+1}) + \left\langle x - x^{k+1}, \frac{x^{k+1} - x^k}{c^{k+1}} + G\left( z^{k+1}, \bar{x}^{k+1}, \bar{\lambda}^{k+1}; \xi_2^{k+1}\right)\right\rangle & \geq 0. \label{opt-X}
\end{align}
Setting $x = x^{k+1}$ in \eqref{opt-barX} and $x = \bar{x}^{k+1}$ in \eqref{opt-X}, and summing two resulting inequalities yields that
\begin{eqnarray*}
& & \frac{1}{c^{k+1}}\left\|  x^{k+1} - \bar{x}^{k+1} \right\|^2 \\
& \leq & \left\langle x^{k+1} - \bar{x}^{k+1}, G\left(z^{k+1},x^k,\lambda^{k}; \xi_1^{k+1}\right) - G\left(z^{k+1},\bar{x}^{k+1}, \bar{\lambda}^{k+1}; \xi_2^{k+1}\right)\right\rangle \\
& \leq & \left\| x^{k+1} - \bar{x}^{k+1} \right\| \left\| G\left(z^{k+1},x^k,\lambda^{k}; \xi_1^{k+1}\right) - G\left(z^{k+1},\bar{x}^{k+1}, \bar{\lambda}^{k+1}; \xi_2^{k+1}\right) \right\|,
\end{eqnarray*}
which implies that
\begin{equation}\label{x-G-bounded}
\left\|  x^{k+1} - \bar{x}^{k+1} \right\| \leq c^{k+1}\left\| G\left(z^{k+1},x^k,\lambda^{k}; \xi_1^{k+1}\right) - G\left(z^{k+1},\bar{x}^{k+1}, \bar{\lambda}^{k+1}; \xi_2^{k+1}\right) \right\|.  
\end{equation}
Therefore, we get
\begin{eqnarray}\label{inequality-x-xbar}
& & r_1(x^{k+1}) - r_1(\bar{x}^{k+1}) + \left\langle x^{k+1} - \bar{x}^{k+1}, G\left(z^{k+1},\bar{x}^{k+1}, \bar{\lambda}^{k+1}; \xi_2^{k+1}\right) \right\rangle \\
& \geq & \left\langle x^{k+1} - \bar{x}^{k+1},  G\left(z^{k+1},\bar{x}^{k+1}, \bar{\lambda}^{k+1}; \xi_2^{k+1}\right) - G\left(z^{k+1},x^k,\lambda^{k}; \xi_1^{k+1}\right) \right\rangle \nonumber \\
& & - \left\langle x^{k+1} - \bar{x}^{k+1}, \frac{\bar{x}^{k+1} - x^k}{c^{k+1}}\right\rangle \nonumber \\
& \geq & -c^{k+1}\left\| G\left(z^{k+1},\bar{x}^{k+1}, \bar{\lambda}^{k+1}; \xi_2^{k+1}\right) - G\left(z^{k+1},x^k,\lambda^{k}; \xi_1^{k+1}\right) \right\|^2 \nonumber \\
& & - \frac{1}{2c^{k+1}}\left\| x^{k+1} - x^k\right\|^2 + \frac{1}{2c^{k+1}}\left\| x^{k+1} - \bar{x}^{k+1} \right\|^2 + \frac{1}{2c^{k+1}}\left\| \bar{x}^{k+1} - x^k \right\|^2.  \nonumber 
\end{eqnarray}
where the first inequality is obtained by letting $x=x^{k+1}$ in \eqref{opt-barX} and the second inequality follows from \eqref{x-G-bounded}. Furthermore, we have 
\begin{eqnarray}\label{G-x-lambda-bounded}
& & \left\| G\left(z^{k+1},\bar{x}^{k+1}, \bar{\lambda}^{k+1}; \xi_2^{k+1}\right) - G\left(z^{k+1},x^k,\lambda^{k}; \xi_1^{k+1}\right) \right\|^2 \\
& = & \left\| \bar{\delta}^{k+1} + \nabla l(\bar{x}^{k+1}) - F^\top\bar{\lambda}^{k+1}  - \left[ \delta^{k+1} + \nabla l(x^k) - F^\top\lambda^k \right]\right\|^2 \nonumber \\
& \leq & 4\left\|\delta^{k+1}\right\|^2 + 4\left\|\bar{\delta}^{k+1}\right\|^2 +  4L^2\left\| x^k - \bar{x}^{k+1}\right\|^2 + 4\sigma_{\max}(F^\top F)\left\| \lambda^k - \bar{\lambda}^{k+1}\right\|^2,    \nonumber
\end{eqnarray}
where $\delta^{k+1}$ and $\bar{\delta}^{k+1}$ are defined in \eqref{def-delta}. By substituting \eqref{G-x-lambda-bounded} into \eqref{inequality-x-xbar}, and then summing the resulting inequality and \eqref{opt-X}, we have
\begin{eqnarray}\label{inequality-X-gap}
& & r_1(x) - r_1(\bar{x}^{k+1}) + \left\langle x - \bar{x}^{k+1}, G\left(z^{k+1},\bar{x}^{k+1}, \bar{\lambda}^{k+1}; \xi_2^{k+1}\right)\right\rangle \\
& \geq & -4c^{k+1}\left\|\delta^{k+1}\right\|^2 - 4c^{k+1}\left\|\bar{\delta}^{k+1}\right\|^2 - 4c^{k+1}\sigma_{\max}(F^\top F)\left\| \lambda^k - \bar{\lambda}^{k+1}\right\|^2 \nonumber \\
& & - 4c^{k+1}L^2\left\| x^k - \bar{x}^{k+1}\right\|^2 - \frac{1}{2c^{k+1}}\left\| x^{k+1} - x^k\right\|^2 + \frac{1}{2c^{k+1}}\left\| x^{k+1} - \bar{x}^{k+1} \right\|^2 \nonumber \\
& & + \frac{1}{2c^{k+1}}\left\| \bar{x}^{k+1} - x^k \right\|^2 - \left\langle x - x^{k+1}, \frac{x^{k+1} - x^k}{c^{k+1}} \right\rangle  \nonumber \\ 
& = & -4c^{k+1}\left\|\delta^{k+1}\right\|^2 - 4c^{k+1}\left\|\bar{\delta}^{k+1}\right\|^2 - 4c^{k+1}\sigma_{\max}(F^\top F)\left\| \lambda^k - \bar{\lambda}^{k+1}\right\|^2 \nonumber \\
& & + \frac{1}{2c^{k+1}}\left\| x^{k+1} - \bar{x}^{k+1} \right\|^2 + \left[ \frac{1}{2c^{k+1}} - 4c^{k+1}L^2\right]\left\| x^k - \bar{x}^{k+1}\right\|^2   \nonumber \\
& & - \frac{1}{2c^{k+1}}\left\| x - x^k \right\|^2 + \frac{1}{2c^{k+1}}\left\| x - x^{k+1} \right\|^2.   \nonumber
\end{eqnarray}
On the other hand, we have
\begin{eqnarray}\label{inequality-Lambda-gap}
& & \left\langle\lambda - \bar{\lambda}^{k+1}, F\bar{x}^{k+1} - z^{k+1}\right\rangle \\
& = & \frac{1}{\gamma}\left\langle\lambda - \lambda^{k+1} + \lambda^{k+1} - \bar{\lambda}^{k+1}, \lambda^k - \lambda^{k+1}\right\rangle \nonumber \\
& = & -\frac{1}{2\gamma}\left\|\lambda - \lambda^k\right\|^2 + \frac{1}{2\gamma}\left\|\lambda - \lambda^{k+1}\right\|^2 - \frac{1}{2\gamma}\left\|\lambda^{k+1} - \bar{\lambda}^{k+1}\right\|^2 + \frac{1}{2\gamma}\left\|\lambda^k - \bar{\lambda}^{k+1}\right\|^2 \nonumber \\
& \geq &  -\frac{1}{2\gamma}\left\|\lambda - \lambda^k\right\|^2 + \frac{1}{2\gamma}\left\|\lambda - \lambda^{k+1}\right\|^2 + \frac{1}{2\gamma}\left\|\lambda^k - \bar{\lambda}^{k+1}\right\|^2 - \frac{\gamma\sigma_{\max}(F^\top F)}{2}\left\| x^k - \bar{x}^{k+1}\right\|^2 \nonumber,
\end{eqnarray}
where the last inequality holds since
\begin{displaymath}
\lambda^{k+1} - \bar{\lambda}^{k+1} = \gamma\left( F\bar{x}^{k+1} - z^{k+1}\right) - \gamma\left( Fx^k - z^{k+1}\right) = \gamma \left( F\bar{x}^{k+1} - Fx^k\right). 
\end{displaymath}
Finally, combining \eqref{inequality-Z-gap}, \eqref{inequality-X-gap} and \eqref{inequality-Lambda-gap} yields \eqref{inequality-noisy-optimality}.
\end{proof}

\subsection{Proof of Theorem \ref{Theorem-Convex-Average}}
\begin{restatable}{lemma}{pLemmaConvexAverage}
\label{Lemma-Convex-Average}
Suppose that $\left\{z^{k+1}, \bar{x}^{k+1},\bar{\lambda}^{k+1}, x^{k+1}, \lambda^{k+1}\right\}$ are generated by the SPDPEG algorithm, and $\alpha^{k+1}$ and $c^{k+1}$ are defined in the main paper. For any optimal solution $(z^*, x^*, \lambda^*)$, it holds that
\begin{eqnarray}\label{inequality-convex-average}
& & l(x^*) + r_1(x^*) + r_2(z^*) - \ex [l(\bar{x}^{k+1})] - \ex [r_1(\bar{x}^{k+1})] - \ex[r_2(z^{k+1})] \\
& & + \ex\left[ \left( \begin{array}{c} z^* - z^{k+1} \\ x^* - \bar{x}^{k+1} \\ \lambda - \bar{\lambda}^{k+1} \end{array} \right)^\top\left(\begin{array}{c} \bar{\lambda}^{k+1} \\ -F^\top\bar{\lambda}^{k+1} \\ F\bar{x}^{k+1} - z^{k+1}\end{array}\right)\right] \nonumber \\
& \geq & \frac{\sqrt{k+1}+\tilde{L}}{2}\ex\left\| x^* - x^{k+1} \right\|^2 - \frac{\sqrt{k+1}+\tilde{L}}{2}\ex\left\| x^* - x^k \right\|^2 - \frac{8\sigma^2}{\sqrt{k+1}} \nonumber \\
& & - \frac{1}{2\gamma}\ex\left\|\lambda - \lambda^k\right\|^2 + \frac{1}{2\gamma}\ex\left\|\lambda - \lambda^{k+1}\right\|^2. \nonumber
\end{eqnarray}
\end{restatable}
\begin{proof}
By the definition of $\tilde{L}$, we have $\frac{1}{2\gamma} -4c^{k+1}\sigma_{\max}(F^\top F) \geq 0$ and $\frac{1}{2c^{k+1}} - \frac{\gamma\sigma_{\max}(F^\top F)}{2} - 4c^{k+1}L^2 \geq 0$. Plugging them into \eqref{inequality-noisy-optimality} yields that 
\begin{eqnarray*}
& & r_1(x) + r_2(z) - r_1(\bar{x}^{k+1}) - r_2(z^{k+1}) + \left( \begin{array}{c} z - z^{k+1} \\ x - \bar{x}^{k+1} \\ \lambda - \bar{\lambda}^{k+1} \end{array} \right)^\top\left(\begin{array}{c} \bar{\lambda}^{k+1} \\ G\left( z^{k+1}, \bar{x}^{k+1}, \bar{\lambda}^{k+1}; \xi_2^{k+1} \right) \\ F\bar{x}^{k+1} - z^{k+1}\end{array}\right) \\
& \geq & \frac{1}{2c^{k+1}}\left\| x - x^{k+1} \right\|^2 - \frac{1}{2c^{k+1}}\left\| x - x^k \right\|^2 - 4c^{k+1}\left\|\delta^{k+1}\right\|^2 - 4c^{k+1}\left\|\bar{\delta}^{k+1}\right\|^2 \\
& & - \frac{1}{2\gamma}\left\|\lambda - \lambda^k\right\|^2 + \frac{1}{2\gamma}\left\|\lambda - \lambda^{k+1}\right\|^2. 
\end{eqnarray*}
Moreover, we have
\begin{eqnarray*}
& & \left( x - \bar{x}^{k+1}\right)^\top G\left( z^{k+1}, \bar{x}^{k+1}, \bar{\lambda}^{k+1}; \xi_2^{k+1}\right) \\
& = & \left( x - \bar{x}^{k+1}\right)^\top\nabla l(\bar{x}^{k+1}) + \left( x - \bar{x}^{k+1}\right)^\top\bar{\delta}^{k+1} + \left( x - \bar{x}^{k+1}\right)^\top\left[- F^\top\bar{\lambda}^{k+1}\right] \\
& \leq & l(x) - l(\bar{x}^{k+1}) + \left( x - \bar{x}^{k+1}\right)^\top\left[- F^\top\bar{\lambda}^{k+1}\right] + \left( x - \bar{x}^{k+1}\right)^\top\bar{\delta}^{k+1}.
\end{eqnarray*}
Therefore, we conclude that
\begin{eqnarray*}
& & l(x) + r_1(x) + r_2(z) - l(\bar{x}^{k+1}) - r_1(\bar{x}^{k+1}) - r_2(z^{k+1}) + \left( \begin{array}{c} z - z^{k+1} \\ x - \bar{x}^{k+1} \\ \lambda - \bar{\lambda}^{k+1} \end{array} \right)^\top\left(\begin{array}{c} \bar{\lambda}^{k+1} \\ -F^\top\bar{\lambda}^{k+1} \\ F\bar{x}^{k+1} - z^{k+1}\end{array}\right) \\
& \geq & \frac{\sqrt{k+1}+\tilde{L}}{2}\left\| x - x^{k+1} \right\|^2 - \frac{\sqrt{k+1}+\tilde{L}}{2}\left\| x - x^k \right\|^2 - \frac{4}{\sqrt{k+1}+\tilde{L}} \left( \left\|\delta^{k+1}\right\|^2 +\left\|\bar{\delta}^{k+1}\right\|^2 \right) \\
& & - \frac{1}{2\gamma}\left\|\lambda - \lambda^k\right\|^2 + \frac{1}{2\gamma}\left\|\lambda - \lambda^{k+1}\right\|^2 - \left( x - \bar{x}^{k+1}\right)^\top\bar{\delta}^{k+1}. 
\end{eqnarray*}
Since $x^k$ and $\bar{x}^{k+1}$ are independent of $\xi_1^{k+1}$ and $\xi_2^{k+1}$ respectively, we take the expectation on both sides of above inequality conditioning on $\xi_2^{k+1}$, then $\xi_1^{k+1}$, and then $\{\xi_1^j,\xi_2^j\}_{j\leq k}$. Finally, we set $(z,x)=(z^*,x^*)$, and conclude \eqref{inequality-convex-average}.
\end{proof}
We are ready to prove Theorem \ref{Theorem-Convex-Average}. For any $\lambda\in\br^p$, we have
{\small
\begin{eqnarray}\label{theorm-inequality-convex-average}
& & l(x^*) + r_1(x^*) + r_2(z^*) - \ex [l(\tilde{x}^t)] - \ex [r_1(\tilde{x}^t)] - \ex[r_2(\tilde{z}^t)] + \lambda^\top\left(F\ex[\tilde{x}^t] - \ex[\tilde{z}^t]\right) \nonumber \\ 
& = & l(x^*) + r_1(x^*) + r_2(z^*) - \ex [l(\tilde{x}^t)] - \ex [r_1(\tilde{x}^t)] - \ex[r_2(\tilde{z}^t)] + \ex\left[ \left( \begin{array}{c} z^* - \tilde{z}^t \\ x^* - \tilde{x}^t \\ \lambda - \tilde{\lambda}^t \end{array} \right)^\top\left(\begin{array}{c} \tilde{\lambda}^t \\ -F^\top\tilde{\lambda}^t \\ F\tilde{x}^t - \tilde{z}^t\end{array}\right)\right] \nonumber \\ 
& \geq & \frac{1}{t+1}\sum_{k=0}^t \left\{ l(x^*) + r_1(x^*) + r_2(z^*) - \ex [l(\bar{x}^{k+1})] - \ex[r_1(\bar{x}^{k+1})] - \ex[r_2(z^{k+1})] \right. \nonumber \\
& & \left. + \ex\left[ \left( \begin{array}{c} z^* - z^{k+1} \\ x^* - \bar{x}^{k+1} \\ \lambda - \bar{\lambda}^{k+1} \end{array} \right)^\top\left(\begin{array}{c} \bar{\lambda}^{k+1} \\ -F^\top\bar{\lambda}^{k+1} \\ F\bar{x}^{k+1} - z^{k+1} \end{array}\right)\right] \right\} \nonumber \\
& \geq & \frac{1}{t+1}\sum_{k=0}^t \left[ \frac{\sqrt{k+1}+\tilde{L}}{2}\ex\left\| x^* - x^{k+1} \right\|^2 - \frac{\sqrt{k+1}+\tilde{L}}{2}\ex\left\| x^* - x^k \right\|^2 - \frac{8\sigma^2}{\sqrt{k+1}} \right. \nonumber \\
& & \left. - \frac{1}{2\gamma}\ex\left\|\lambda - \lambda^k\right\|^2 + \frac{1}{2\gamma}\ex\left\|\lambda - \lambda^{k+1}\right\|^2\right] \nonumber \\
& \geq & - \frac{\tilde{L}}{2(t+1)}\left\| x^* - x^0\right\|^2 - \frac{D_x^2 + 16\sigma^2}{2\sqrt{t+1}} - \frac{1}{2\gamma(t+1)}\left\|\lambda - \lambda^0\right\|^2, 
\end{eqnarray}
}where the first inequality holds due to the convexity of $l$, $r_1$ and $r_2$. Note that the optimality condition imply the following inequality
\begin{equation}\label{inequality-opt-problem-1}
0 \geq l(x^*) + r_1(x^*) + r_2(z^*) - \ex [l(\tilde{x}^t)] - \ex [r_1(\tilde{x}^t)] - \ex[r_2(\tilde{z}^t)] + (\lambda^*)^\top\left( F\ex[\tilde{x}^t] - \ex[\tilde{z}^t]\right). 
\end{equation}
Now, define $\rho := \left\|\lambda^*\right\|+1$. By using Cauchy-Schwarz inequality in \eqref{inequality-opt-problem-1}, we obtain
\begin{equation}\label{inequality-opt-problem-2}
0 \leq \ex [l(\tilde{x}^t)] + \ex [r_1(\tilde{x}^t)] + \ex[r_2(\tilde{z}^t)] - l(x^*) - r_1(x^*) - r_2(z^*) + \rho\left\| F\ex[\tilde{x}^t] - \ex[\tilde{z}^t] \right\|. 
\end{equation}
By setting $\lambda = -\rho\left(F\ex[\tilde{x}^t] - \ex[\tilde{z}^t]\right)/\left\| F\ex[\tilde{x}^t] - \ex[\tilde{z}^t]\right\|$ in \eqref{theorm-inequality-convex-average}, and noting that $\left\|\lambda\right\| = \rho$, we obtain
\begin{eqnarray}\label{inequality-opt-problem-3}
& & \ex [l(\tilde{x}^t)] + \ex [r_1(\tilde{x}^t)] + \ex[r_2(\tilde{z}^t)] - l(x^*) - r_1(x^*) - r_2(z^*) + \rho\left\| F\ex[\tilde{x}^t] - \ex[\tilde{z}^t] \right\| \nonumber \\
& \leq & \frac{\tilde{L}D_x^2}{2(t+1)} + \frac{D_x^2 + 16\sigma^2}{2\sqrt{t+1}} + \frac{\rho^2 + \left\|\lambda^0\right\|^2}{\gamma(t+1)}. 
\end{eqnarray}
We now define the function
\begin{displaymath}
v(\eta) = \min\left\{ l(x) + r_1(x) + r_2(z) | Fx - z = \eta, x\in\XCal\right\}.
\end{displaymath}
It is easy to verify that $v$ is convex, $v(0) = l(x^*)+r_1(x^*)+ r_2(z^*)$, and $\lambda^*\in\partial v(0)$. Therefore, from the
convexity of $v$, it holds that
\begin{equation}\label{v-convex}
v(\eta)\geq v(0) + \left\langle\lambda^*,\eta\right\rangle \geq l(x^*) + r_1(x^*) + r_2(z^*) - \left\|\lambda^*\right\|\left\|\eta\right\|. 
\end{equation}
Let $\bar{\eta} = F\ex[\tilde{x}^t]-\ex[\tilde{z}^t]$, we have 
\begin{displaymath}
\ex [l(\tilde{x}^t)] + \ex [r_1(\tilde{x}^t)] + \ex[r_2(\tilde{z}^t)] \geq l(\ex[\tilde{x}^t]) + r_1(\ex[\tilde{x}^t]) + r_2(\ex[\tilde{z}^t]) \geq v(\bar{\eta}). 
\end{displaymath}
Therefore, combining \eqref{inequality-opt-problem-2}, \eqref{inequality-opt-problem-3} and \eqref{v-convex}, we get
\begin{eqnarray*}
-\left\|\lambda^*\right\|\left\|\bar{\eta}\right\| & \leq & \ex [l(\tilde{x}^t)] + \ex [r_1(\tilde{x}^t)] + \ex[r_2(\tilde{z}^t)] - l(x^*) - r_1(x^*) - r_2(z^*) \\
& \leq & \frac{\tilde{L}D_x^2}{2(t+1)} + \frac{D_x^2 + 16\sigma^2}{2\sqrt{t+1}} + \frac{\rho^2 + \left\|\lambda^0\right\|^2}{\gamma(t+1)} - \rho\left\|\bar{\eta}\right\|,
\end{eqnarray*}
which implies \eqref{result-convex-average-1} and \eqref{result-convex-average-2}.

\subsection{Proof of Theorem \ref{Theorem-Strongly-Convex-Average}}
\begin{restatable}{lemma}{pLemmaStronglyConvexAverage}
\label{Lemma-Strongly-Convex-Average}
Let $\left\{z^{k+1}, \bar{x}^{k+1},\bar{\lambda}^{k+1},x^{k+1},\lambda^{k+1}\right\}$ be generated by the SPDPEG Algorithm, and $\alpha^{k+1}$ and $c^{k+1}$ be defined in the main paper. For any optimal solution $\left(z^*, x^*\right)$, it holds that
\begin{eqnarray}\label{inequality-strongly-convex-average}
& & l(x^*) + r_1(x^*) + r_2(z^*) - \ex [l(\bar{x}^{k+1})] - \ex [r_1(\bar{x}^{k+1})] - \ex[r_2(z^{k+1})] \nonumber \\
& & + \ex\left[ \left( \begin{array}{c} z^* - z^{k+1} \\ x^* - \bar{x}^{k+1} \\ \lambda - \bar{\lambda}^{k+1} \end{array} \right)^\top\left(\begin{array}{c} \bar{\lambda}^{k+1} \\ -F^\top\bar{\lambda}^{k+1} \\ F\bar{x}^{k+1} - z^{k+1}\end{array}\right)\right] \nonumber \\
& \geq & \frac{\mu(k+2)+2\tilde{L}}{4}\ex\left\| x^* - x^{k+1} \right\|^2 - \frac{\mu(k+1)+2\tilde{L}}{4}\ex\left\| x^* - x^k \right\|^2 - \frac{16\sigma^2}{\mu(k+1)} \nonumber \\
& & - \frac{1}{2\gamma}\ex\left\|\lambda - \lambda^k\right\|^2 + \frac{1}{2\gamma}\ex\left\|\lambda - \lambda^{k+1}\right\|^2.
\end{eqnarray}
\end{restatable}
\begin{proof}
Since $\frac{1}{2\gamma} -4c^{k+1}\sigma_{\max}(F^\top F) \geq 0$ and $\frac{1}{2c^{k+1}} - \frac{\gamma\sigma_{\max}(F^\top F)}{2} - 4c^{k+1}L^2 \geq 0$ and $c^{k+1} < \frac{1}{\mu}$, we conclude from \eqref{inequality-noisy-optimality} that 
\begin{eqnarray*}
& & r_1(x) + r_2(z) - r_1(\bar{x}^{k+1}) - r_2(z^{k+1}) + \left( \begin{array}{c} z - z^{k+1} \\ x - \bar{x}^{k+1} \\ \lambda - \bar{\lambda}^{k+1} \end{array} \right)^\top\left(\begin{array}{c} \bar{\lambda}^{k+1} \\ G\left( z^{k+1}, \bar{x}^{k+1}, \bar{\lambda}^{k+1}; \xi_2^{k+1}\right) \\ F\bar{x}^{k+1} - z^{k+1}\end{array}\right) \\
& \geq & \frac{1}{2c^{k+1}}\left\| x - x^{k+1} \right\|^2 - \frac{1}{2c^{k+1}}\left\| x - x^k \right\|^2 - 4c^{k+1}\left\|\delta^{k+1}\right\|^2 - 4c^{k+1}\left\|\bar{\delta}^{k+1}\right\|^2 \\
& & - \frac{1}{2\gamma}\left\|\lambda - \lambda^k\right\|^2 + \frac{1}{2\gamma}\left\|\lambda - \lambda^{k+1}\right\|^2 + \frac{\mu}{2}\left\| \bar{x}^{k+1} - x^{k+1}\right\|^2. 
\end{eqnarray*}
Moreover, we have
\begin{eqnarray*}
& & \left( x - \bar{x}^{k+1}\right)^\top G\left( z^{k+1}, \bar{x}^{k+1}, \bar{\lambda}^{k+1}; \xi_2^{k+1}\right) \\
& = & \left( x - \bar{x}^{k+1}\right)^\top\nabla l(\bar{x}^{k+1}) + \left( x - \bar{x}^{k+1}\right)^\top\bar{\delta}^{k+1} + \left( x - \bar{x}^{k+1}\right)^\top\left[- A^\top\bar{\lambda}^{k+1}\right] \\
& \leq & l(x) - l(\bar{x}^{k+1}) - \frac{\mu}{2}\left\| x - \bar{x}^{k+1}\right\|^2 + \left( x - \bar{x}^{k+1}\right)^\top\left[- A^\top\bar{\lambda}^{k+1}\right] + \left( x - \bar{x}^{k+1}\right)^\top\bar{\delta}^{k+1}.
\end{eqnarray*}
Therefore, we conclude that
{\small
\begin{eqnarray*}
& & l(x) + r_1(x) + r_2(z) - l(\bar{x}^{k+1}) - r_1(\bar{x}^{k+1}) - r_2(z^{k+1}) + \left( \begin{array}{c} z - z^{k+1} \\ x - \bar{x}^{k+1} \\ \lambda - \bar{\lambda}^{k+1} \end{array} \right)^\top\left(\begin{array}{c} \bar{\lambda}^{k+1} \\ -F^\top\bar{\lambda}^{k+1} \\ F\bar{x}^{k+1} - z^{k+1}\end{array}\right) \\
& \geq & \frac{\mu(k+1)+2\tilde{L}}{4}\left\| x - x^{k+1} \right\|^2 - \frac{\mu(k+1)+2\tilde{L}}{4}\left\| x - x^k \right\|^2 - \frac{8}{\mu(k+1)+\tilde{L}}\left( \left\|\delta^{k+1}\right\|^2 + \left\|\bar{\delta}^{k+1}\right\|^2\right) \\
& & - \frac{1}{2\gamma}\left\|\lambda - \lambda^k\right\|^2 + \frac{1}{2\gamma}\left\|\lambda - \lambda^{k+1}\right\|^2 - \left( x - \bar{x}^{k+1}\right)^\top\bar{\delta}^{k+1} + \frac{\mu}{2}\left\| x - \bar{x}^{k+1}\right\|^2 + \frac{\mu}{2}\left\| \bar{x}^{k+1} - x^{k+1}\right\|^2 \\
& \geq & \frac{\mu(k+2)+2\tilde{L}}{4}\left\| x - x^{k+1} \right\|^2 - \frac{\mu(k+1)+2\tilde{L}}{4}\left\| x - x^k \right\|^2 - \frac{8}{\mu(k+1)+2\tilde{L}}\left( \left\|\delta^{k+1}\right\|^2 + \left\|\bar{\delta}^{k+1}\right\|^2\right) \\
& & - \frac{1}{2\gamma}\left\|\lambda - \lambda^k\right\|^2 + \frac{1}{2\gamma}\left\|\lambda - \lambda^{k+1}\right\|^2 - \left( x - \bar{x}^{k+1}\right)^\top\bar{\delta}^{k+1}. 
\end{eqnarray*}
}Since $x^k$ and $\bar{x}^{k+1}$ are independent of $\xi_1^{k+1}$ and $\xi_2^{k+1}$ respectively, we take the expectation on both sides of above inequality conditioning on $\xi_2^{k+1}$, then $\xi_1^{k+1}$ and then $\{\xi_1^j,\xi_2^j\}_{j\leq k}$. Finally, we set $(z,x)=(z^*,x^*)$, and conclude \eqref{inequality-strongly-convex-average}.
\end{proof}
We are ready to prove Theorem \ref{Theorem-Strongly-Convex-Average}. For any $\lambda\in\br^p$, we have
\begin{eqnarray*}
& & l(x^*) + r_1(x^*) + r_2(z^*) - \ex [l(\tilde{x}^t)] - \ex [r_1(\tilde{x}^t)] - \ex[r_2(\tilde{z}^t)] + \lambda^\top\left( F\ex[\tilde{x}^t] - \ex[\tilde{z}^t]\right)\\ 
& \geq & \frac{1}{t+1}\sum_{k=0}^t \left\{ l(x^*) + r_1(x^*) + r_2(z^*) - \ex [l(\bar{x}^{k+1})] - \ex [r_1(\bar{x}^{k+1})] - \ex[r_2(z^{k+1})]\right. \\ 
& & \left. + \ex\left[ \left( \begin{array}{c} z^* - z^{k+1} \\ x^* - \bar{x}^{k+1} \\ \lambda - \bar{\lambda}^{k+1} \end{array} \right)^\top\left(\begin{array}{c} \bar{\lambda}^{k+1} \\ -F^\top\bar{\lambda}^{k+1} \\ F\bar{x}^{k+1} - z^{k+1}\end{array}\right)\right] \right\} \nonumber \\
& \geq & \frac{1}{t+1}\sum_{k=0}^t \left[ \frac{\mu(k+2)+2\tilde{L}}{4}\ex\left\| x^* - x^{k+1} \right\|^2 - \frac{\mu(k+1)+2\tilde{L}}{4}\ex\left\| x^* - x^k \right\|^2 - \frac{16\sigma^2}{\mu(k+1)}\right. \nonumber \\
& & \left. - \frac{1}{2\gamma}\ex\left\|\lambda - \lambda^k\right\|^2 + \frac{1}{2\gamma}\ex\left\|\lambda - \lambda^{k+1}\right\|^2\right] \nonumber \\
& \geq & - \frac{\mu+2\tilde{L}}{4(t+1)}\left\| x^* - x^0\right\|^2 - \frac{16\sigma^2\log(t+1)}{\mu(t+1)} - \frac{1}{2\gamma(t+1)}\left\|\lambda - \lambda^0\right\|^2. \nonumber  
\end{eqnarray*}
where the first inequality holds due to the convexity of $l$, $r_1$ and $r_2$. By the same argument as Theorem \ref{Theorem-Convex-Average}, we conclude \eqref{result-strongly-convex-average-1} and \eqref{result-strongly-convex-average-2}.

\subsection{Proof of Theorem \ref{Theorem-Strongly-Convex-Nonaverage}}
\begin{restatable}{lemma}{pLemmaStronglyConvexNonAverage}
\label{Lemma-Strongly-Convex-Non-Average}
Let $\left\{z^{k+1}, \bar{x}^{k+1},\bar{\lambda}^{k+1}, x^{k+1}, \lambda^{k+1}\right\}$ be generated by the SPDPEG Algorithm, and $\alpha^{k+1}$ and $c^{k+1}$ be defined in the main paper. For any optimal solution $\left(z^*, x^*\right)$, it holds that
\begin{eqnarray}\label{inequality-strongly-convex-non-average}
& & l(x^*) + r_1(x^*) + r_2(z^*) - \ex [l(\bar{x}^{k+1})] - \ex [r_1(\bar{x}^{k+1})] - \ex[r_2(z^{k+1})] \nonumber \\
& & + \ex\left[ \left( \begin{array}{c} z^* - z^{k+1} \\ x^* - \bar{x}^{k+1} \\ \lambda - \bar{\lambda}^{k+1} \end{array} \right)^\top\left(\begin{array}{c} \bar{\lambda}^{k+1} \\ -F^\top\bar{\lambda}^{k+1} \\ F\bar{x}^{k+1} - z^{k+1}\end{array}\right)\right] \nonumber \\
& \geq & \frac{\mu(k+4)+4\tilde{L}}{8}\ex\left\| x^* - x^{k+1} \right\|^2 - \frac{\mu(k+2)+4\tilde{L}}{8}\ex\left\| x^* - x^k \right\|^2 - \frac{32\sigma^2}{\mu(k+2)} \nonumber \\
& & - \frac{1}{2\gamma}\ex\left\|\lambda - \lambda^k\right\|^2 + \frac{1}{2\gamma}\ex\left\|\lambda - \lambda^{k+1}\right\|^2.
\end{eqnarray}
\end{restatable}
\begin{proof}
By the same argument as Lemma \ref{Lemma-Strongly-Convex-Average}, we conclude from \eqref{inequality-noisy-optimality} that
{\small 
\begin{eqnarray*}
& & r_1(x) + r_2(z) - r_1(\bar{x}^{k+1}) - r_2(z^{k+1}) + \left( \begin{array}{c} z - z^{k+1} \\ x - \bar{x}^{k+1} \\ \lambda - \bar{\lambda}^{k+1} \end{array} \right)^\top\left(\begin{array}{c} \bar{\lambda}^{k+1} \\ G\left( z^{k+1}, \bar{x}^{k+1},  \bar{\lambda}^{k+1}; \xi_2^{k+1}\right) \\ F\bar{x}^{k+1} - z^{k+1}\end{array}\right) \\
& \geq & \frac{1}{2c^{k+1}}\left\| x - x^{k+1} \right\|^2 - \frac{1}{2c^{k+1}}\left\| x - x^k \right\|^2 - 4c^{k+1}\left\|\delta^{k+1}\right\|^2 - 4c^{k+1}\left\|\bar{\delta}^{k+1}\right\|^2 \\
& & - \frac{1}{2\gamma}\left\|\lambda - \lambda^k\right\|^2 + \frac{1}{2\gamma}\left\|\lambda - \lambda^{k+1}\right\|^2 + \frac{\mu}{2}\left\| \bar{x}^{k+1} - x^{k+1}\right\|^2,  
\end{eqnarray*}
}and
\begin{eqnarray*}
& & \left( x - \bar{x}^{k+1}\right)^\top G\left( z^{k+1}, \bar{x}^{k+1}, \bar{\lambda}^{k+1}; \xi_2^{k+1}\right) \\
& \leq & l(x) - l(\bar{x}^{k+1}) - \frac{\mu}{2}\left\| x - \bar{x}^{k+1}\right\|^2 + \left( x - \bar{x}^{k+1}\right)^\top\left[- A^\top\bar{\lambda}^{k+1}\right] + \left( x - \bar{x}^{k+1}\right)^\top\bar{\delta}^{k+1}.
\end{eqnarray*}
Therefore, we conclude that
\begin{eqnarray*}
& & l(x) + r_1(x) + r_2(z) - l(\bar{x}^{k+1}) - r_1(\bar{x}^{k+1}) - r_2(z^{k+1}) + \left( \begin{array}{c} z - z^{k+1} \\ x - \bar{x}^{k+1} \\ \lambda - \bar{\lambda}^{k+1} \end{array} \right)^\top\left(\begin{array}{c} \bar{\lambda}^{k+1} \\ -F^\top\bar{\lambda}^{k+1} \\ F\bar{x}^{k+1} - z^{k+1} \end{array}\right) \\
& \geq & \frac{\mu(k+4)+4\tilde{L}}{8}\left\| x - x^{k+1} \right\|^2 - \frac{\mu(k+2)+4\tilde{L}}{8}\left\| x - x^k \right\|^2 - \frac{16}{\mu(k+2)+4\tilde{L}}\left( \left\|\delta^{k+1}\right\|^2 + \left\|\bar{\delta}^{k+1}\right\|^2\right) \\
& & - \frac{1}{2\gamma}\left\|\lambda - \lambda^k\right\|^2 + \frac{1}{2\gamma}\left\|\lambda - \lambda^{k+1}\right\|^2 - \left( x - \bar{x}^{k+1}\right)^\top\bar{\delta}^{k+1}. 
\end{eqnarray*}
Since $x^k$ and $\bar{x}^{k+1}$ are independent of $\xi_1^{k+1}$ and $\xi_2^{k+1}$ respectively, we take the expectation on both sides of above inequality conditioning on $\xi_2^{k+1}$, then $\xi_1^{k+1}$ and then $\{\xi_1^j,\xi_2^j\}_{j\leq k}$. Finally, we set $(z, x)=(z^*, x^*)$, and conclude \eqref{inequality-strongly-convex-non-average}.
\end{proof}
We are ready to prove Theorem \ref{Theorem-Strongly-Convex-Nonaverage}. For any $\lambda\in\br^p$, we have
\begin{eqnarray*}
& & l(x^*) + r_1(x^*) + r_2(z^*) - \ex [l(\tilde{x}^t)] - \ex [r_1(\tilde{x}^t)] - \ex[r_2(\tilde{z}^t)] + \lambda^\top\left(F\ex[\tilde{x}^t] - \ex[\tilde{z}^t]\right)\\ 
& \geq & \frac{2}{(t+1)(t+6)}\sum_{k=0}^t (k+3)\left\{ l(x^*) + r_1(x^*) + r_2(z^*) - \ex [l(\bar{x}^{k+1})] - \ex [r_1(\bar{x}^{k+1})] - \ex[r_2(z^{k+1})] \right. \\
& & \left. + \ex\left[ \left( \begin{array}{c} z^* - z^{k+1} \\ x^* - \bar{x}^{k+1} \\ \lambda - \bar{\lambda}^{k+1} \end{array} \right)^\top\left(\begin{array}{c} \bar{\lambda}^{k+1} \\ -F^\top\bar{\lambda}^{k+1} \\ F\bar{x}^{k+1} - z^{k+1}\end{array}\right)\right] \right\} \nonumber \\
& \geq & \frac{2}{(t+1)(t+6)}\sum_{k=0}^t \left(k+3\right)\left[ \frac{\mu(k+4)+4\tilde{L}}{8}\ex\left\| x^* - x^{k+1} \right\|^2 - \frac{\mu(k+2)+4\tilde{L}}{8}\ex\left\| x^* - x^k \right\|^2  \right. \nonumber \\
& & \left. - \frac{32\sigma^2}{\mu(k+2)} - \frac{1}{2\gamma}\ex\left\|\lambda - \lambda^k\right\|^2 + \frac{1}{2\gamma}\ex\left\|\lambda - \lambda^{k+1}\right\|^2\right] \nonumber \\
& \geq & - \frac{3\mu+2\tilde{L}}{4(t+1)(t+6)}\left\| x^* - x^0\right\|^2 - \frac{96\sigma^2}{\mu(t+6)} - \frac{2\left\|\lambda\right\|^2 + 2D_\lambda^2}{\gamma(t+1)} \nonumber,
\end{eqnarray*}
where the first inequality holds due to the convexity of $l$, $r_1$ and $r_2$. By the same argument as Theorem \ref{Theorem-Convex-Average}, we conclude \eqref{result-strongly-convex-nonaverage-1} and \eqref{result-strongly-convex-nonaverage-2}.

\section{Experiment}
We apply our proposed \textsf{SPDPEG} algorithm to solve following two popular problems: fused logistic regression (FLR)~\eqref{Prob:FLR} and graph-guided regularized logistic regression (GGRLR)~\eqref{Prob:GGRLR} \cite{Zhong-2014-Fast}, which are formulated as follows
\begin{equation}\label{Prob:FLR}
\textnormal{FLR:}~~~\min_x \ l(x) + \gamma \|x\|_1 + \lambda\|Lx\|_1,
\end{equation}
and
\begin{equation}\label{Prob:GGRLR}
\textnormal{GGRLR:}~~~\min_x \ l(x) + \frac{\gamma}{2}\left\|x\right\|_2^2 + \lambda\|Fx\|_1.
\end{equation}
Here {\small$l(x) = \frac{1}{N}\left[\sum\limits_{i=1}^N l(x,\xi_i)\right]$}, where $l(x,\xi_i)$ is the logistic loss on $\xi_i$ and $\lambda>0$ is a parameter. $L$ and $F$ are penalty matrices promoting the desired sparse structure of $x$. Specifically, $L\in\mathbb R^{(n-1)\times n}$ in problem~\eqref{Prob:FLR} is specified as a matrix with all ones on the diagonal, negative ones on the super-diagonal and zeros elsewhere, and $F$ in problem~\eqref{Prob:GGRLR} is generated by sparse inverse covariance selection \cite{Scheinberg-2010-Sparse}.

\begin{table}[t]
\renewcommand{\arraystretch}{0.8} 
\caption{Statistics of datasets.}
\label{tab:data}
\begin{center}
\begin{tabular}{c|c|c} \hline
dataset & number of samples & dimensionality \\ \hline 
\textit{splice} & 1000 & 60 \\ 
\textit{svmguide3} & 1243 & 21\\ 
\textit{mushrooms} & 8,124 & 112\\
\textit{a9a} & 32,561 & 123 \\
\textit{w8a} & 64,700 & 300 \\ 
\textit{hitech} & 2,301 & 10,080 \\
\textit{k1b} & 2,340 & 21,839 \\ 
\textit{classic} & 7,094 & 41,681 \\
\hline
\end{tabular}
\end{center}
\end{table}

In the experiments, we compare our \textsf{SPDPEG} algorithm with the \textsf{EGADM} algorithm \cite{Lin-2015-Extragradient} and six existing stochastic ADMM-type algorithms \footnote{We use the implementation of \textsf{SADMM}, \textsf{OPG-ADMM} and \textsf{RDA-ADMM} provided by the authors and two adaptive ADMM according to \cite{Zhao-2015-Adaptive}}: \textsf{SGADM} \cite{Gao-2014-Information}, \textsf{SADMM} \cite{Ouyang-2013-Stochastic}, \textsf{OPG-ADMM} \cite{Suzuki-2013-Dual}, \textsf{RDA-ADMM} \cite{Suzuki-2013-Dual}, and two adaptive SADMM (\textit{i.e.}, \textsf{SADMMdiag} and \textsf{SADMMfull})\cite{Zhao-2015-Adaptive}. We exclude online ADMM \cite{Wang-2012-Online} since \cite{Suzuki-2013-Dual} has shown that \textsf{RDA-ADMM} performs better than online ADMM. \textsf{FSADMM} \cite{Zhong-2014-Fast} is also excluded since it requires storage of all gradients, which results in impractical performance in some complex applications \cite{Johnson-2013-Accelerating}.

\begin{figure}[t]
\includegraphics[width=1\textwidth]{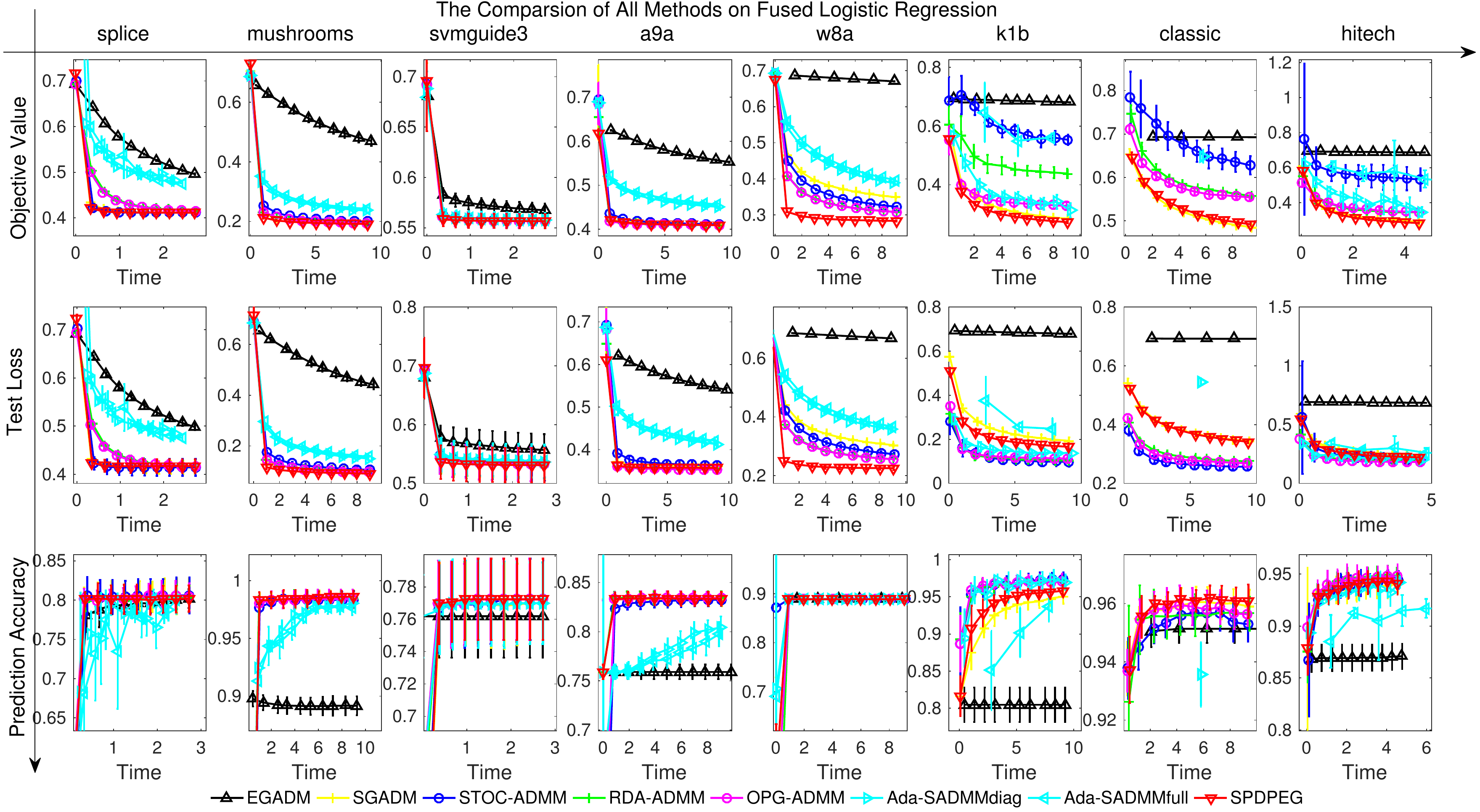}
\caption{Comparison of \textsf{SPDPEG} with \textsf{SGADM}, \textsf{SADMM}, \textsf{RDA-ADMM}, \textsf{OPG-ADMM}, \textsf{SADMMdiag} and \textsf{SADMMfull} on \textbf{Fused Logistic Regression} Task. 
\textbf{First Row}: Average objective values. 
\textbf{Second Row}: Average test losses. 
\textbf{Third Row}: Average prediction accuracies.}
\label{fig-FLR-complexity}
\end{figure}

\begin{figure}[t]
\includegraphics[width=1\textwidth]{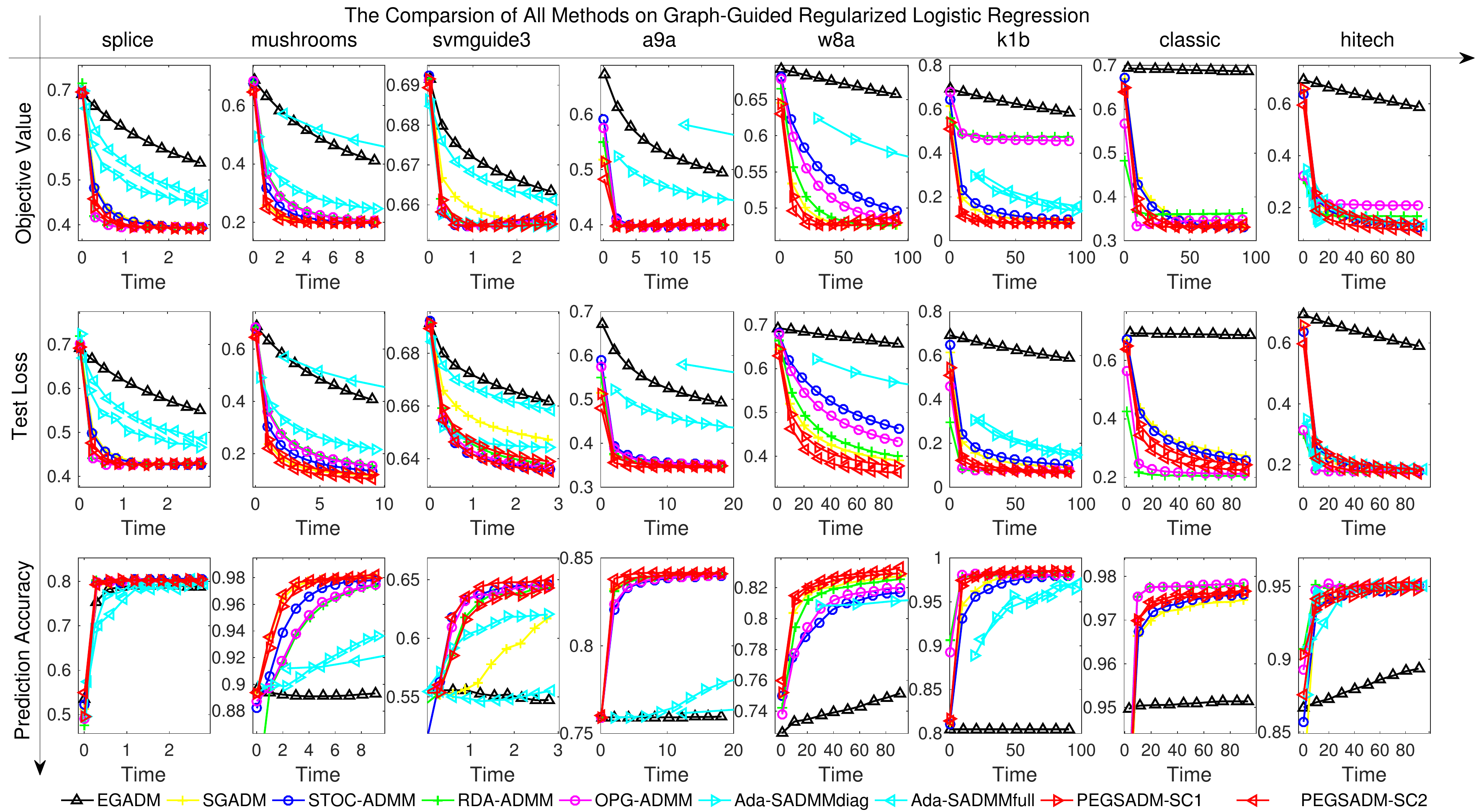}\caption{Comparison of \textsf{SPDPEG-SC1} (Uniformly Averaged) and \textsf{SPDPEG-SC2} (Non-Uniformly Averaged) with \textsf{SGADM}, \textsf{SADMM}, \textsf{RDA-ADMM}, \textsf{OPG-ADMM}, \textsf{SADMMdiag} and \textsf{SADMMfull} on \textbf{Graph-Guided Regularized Logistic Regression} Task. 
\textbf{First Row}: Average objective values. 
\textbf{Second Row}: Average test losses. 
\textbf{Third Row}: Average prediction accuracies.}
\label{fig-GGRLR-complexity}
\end{figure}

The experiments are conducted on five binary classification datasets: \textit{splice}, \textit{svmguide3}, \textit{mushrooms}, \textit{a9a}, and \textit{w8a} \footnote{https://www.csie.ntu.edu.tw/$\sim$cjlin/libsvm/} with large number of samples, 
\textit{classic}, \textit{hitech}, \textit{k1b} \footnote{https://www.shi-zhong.com/software/docdata.zip} with high dimensionality. We set the parameters of SPDPEG exactly following our theory while using the cross validation to select the parameters for other algorithms. For each dataset, we calculate the lipschitz constant $L$ as its classical upper bound $\hat{L} = 0.25\max_{1\leq i\leq n}\|a_i\|^2$. The regularization parameter $\lambda = 5\times 10^{-3}$ and $\gamma = 5\times 10^{-4}$ for problem~\eqref{Prob:FLR}, and $\lambda = 10^{-5}$ and $\gamma = 10^{-2}$ for problem~\eqref{Prob:GGRLR}. To reduce statistical variability, experimental results are repeated 5 rounds. Additionally, we use the metrics including objective value, test loss and prediction accuracy to compare our method with other methods. The ``objective value" means the sum of the loss function and regularized terms evaluated on a training data sample, while the ``test loss" means the value of the loss function evaluated on a test data sample. Specifically, we use objective function values on training datasets, test losses (\textit{i.e.}, $l(x)$) on test datasets, and prediction accuracy on test datasets.

Figure \ref{fig-FLR-complexity} shows the objective value, test loss and prediction accuracy as the functions of the time costs on the FLR task, where the objective function is convex but not necessarily strongly convex.
We observe that our method mostly achieves the best performance, followed by six stochastic ADMM-type algorithms, all of which outperform \textsf{EGADM} by a large margin. We find that the prediction accuracy of the \textsf{SPDPEG} algorithm is competitive with other algorithms, which supports the use of extra-gradient in the \textsf{SPDPEG} algorithm. The performance of our \textsf{SPDPEG} algorithm on six datasets is the most stable and effective among all methods.

We further compare our algorithm with other algorithms on the GGRLR task, where the objective function is strongly convex. We use both uniformly and non-uniformly averaged iterates, noted as \textsf{SPDPEG-SC1} (Uniformly Averaged) and \textsf{SPDPEG-SC2} (Non-Uniformly Averaged). The experimental results presented in Figure \ref{fig-GGRLR-complexity} show that our algorithm consistently outperforms other algorithms, and exhibits the advantage with non-uniformly averaged iterates over its counterpart with uniformly averaged iterates. This matches our analysis in the previous sections.

\section{Conclusions}
In this paper, we proposed a novel algorithm, namely \textsf{Stochastic Primal-Dual Proximal ExtraGradient (SPDPEG)}, to resolve stochastic minimization problems including two regularization terms, one of which is composed with a linear function $F(x)$, as shown in problem~\eqref{prob}. Problem~\eqref{prob} is computationally difficult when the penalty matrix $F$ is non-diagonal or the number of training samples is large.

Inspired by the nice efficiency of \textsf{EGADM}, we developed an ADM-type optimization scheme that employs proximal noisy extra-gradient descent to achieve reasonable numerical efficiency and stability. For general convex objectives, we showed that the uniformly average iterates converge in expectation with the rate of $O(1/\sqrt{t})$; while for strongly convex objectives, the uniformly and non-uniformly average iterates generated by the \textsf{SPDPEG} algorithm were proven to converge in expectation with the $O(\log(t)/t)$ and $O(1/t)$ rates, respectively. It is worth mentioning that these rates are both known to be best possible for first-order stochastic optimization algorithms. The numerical experiments conducted on fused logistic regression and graph-guided regularized logistic regression problems demonstrated that our proposed algorithm consistently outperforms the other competing stochastic algorithms. A future research direction is to consider incorporating variance reduction techniques into the \textsf{SPDPEG} algorithm.

\bibliographystyle{plain}
\bibliography{ref}

\end{document}